\documentclass{article}
\usepackage[T1]{fontenc}
\usepackage{arxiv_preprint,times}

\usepackage{amsmath,amsfonts,bm}

\def\eqref#1{equation~\ref{#1}}

\def\1{\bm{1}}

\DeclareMathAlphabet{\mathsfit}{\encodingdefault}{\sfdefault}{m}{sl}
\SetMathAlphabet{\mathsfit}{bold}{\encodingdefault}{\sfdefault}{bx}{n}

\newcommand{\E}{\mathbb{E}}

\newcommand{\R}{\mathbb{R}}

\DeclareMathOperator*{\argmin}{arg\,min}

\usepackage{url}
\usepackage{amsthm}
\usepackage{graphicx}
\usepackage{booktabs}
\usepackage{float}
\usepackage{flafter}

\usepackage{hyperref}
\hypersetup{colorlinks=true, linkcolor=blue, citecolor=blue, urlcolor=blue}

\newtheorem{proposition}{Proposition}[section]

\newcommand{\Scomp}{S_{\mathrm{comp}}}

\newcommand{\Sphicomp}{S^{\phi_{\mathrm{comp}}}}

\iclrfinalcopy

\hypersetup{
  pdftitle={Compositional Objectives: Learning Structure in Structure},
  pdfauthor={Pranavchandra Vivekananda, Sumukh Bettadapura, Ajan Subramanian}
}

\newcommand{\appref}[1]{\hyperref[#1]{\textit{Appendix~\ref*{#1}}}}
\newcommand{\proofref}[1]{\textit{Proof is in \appref{#1}.}}
\title{Compositional Objectives: \\ Learning Structure in Structure}
\author{Pranavchandra Vivekananda, Sumukh Bettadapura, Ajan Subramanian \\
Kubo Technologies, Inc. \\
\texttt{pranavchandra@kubocreate.com} \\
\texttt{sumukh@kubocreate.com}, \texttt{ajan@kubocreate.com}}

\date{}
\begin{document}

\maketitle


\begin{abstract}
Intelligence is defined in many ways. One of these definitions defines
intelligence as the pursuit of learnable novelty. However, learnable novelty can
be meaningless without the ability to compose the learned structures to take
action and achieve goals. Learnable novelty builds on epiplexity, which is a
way to measure learnable structure in data through a bounded observer. In this paper, we investigate a closed-form spectral
approximation to compute epiplexity. We use a fixed-trace constraint and find
that the epiplexity objective prefers a more uniform distribution of spectral
mass rather than concentrating it in a small number of directions. However, a
representation may spread information across many directions without organizing
that information into features useful for a particular task. To address this
gap, we propose a compositional objective whose observer measures the
relationships between the parts and interactions of an image. We compare it with
the original spectral objective given only the masked parts. In our ImageNet
training runs, the spectral objective with masked parts produces an almost
maximally spread representation while achieving the strongest frozen-feature
classification performance on most evaluations, more than doubling the
linear-probe accuracy of the whole-image baseline. Across multiple image
benchmarks, changing what the observer sees matters more than adding relation
and interaction tokens. At the same time, our prediction-oriented compositional
objective produces substantially better held-out observer prediction but
relatively weaker classification, revealing that spectral
diversity, predictability, and downstream utility are distinct properties. These
results suggest that the usefulness of spectral spreading depends not only on
how much structure is preserved, but on which relationships the observer makes
available to the objective.
\end{abstract}

\section{Introduction}
\label{sec:intro}

Intuitively, we recognize that one part of intelligence is the ability to find
patterns in observations. Another part is the ability to reduce these
observations into efficient, reusable representations. Much of our
understanding of perception, learning and reasoning has placed these two parts
at the center of cognition and machine learning. Information-theoretic
approaches, like Minimum Description Length~\citep{rissanen1978modeling,grunwald2007mdl}, further
formalize the connection between pattern extraction and compression as
extracting structural regularity in data. However, these abilities are
constrained by the computational bounds of the entity or model that is
observing the data. Computationally constrained theories of information
formalize precisely this distinction between what is present in data and what
can be used.

One such formalization is epiplexity~\citep{finzi2026epiplexity}. It is defined as the
description length in bits of the time-bounded MDL-minimizing program.
Intuitively, epiplexity can be thought of as the amount of structure that the
best model extracts from data, given constraints of time and compute bounds.
This separates structural information from the residual, time-bounded entropy.

\citet{zhang2026learnable} introduce the idea of learnable novelty
and make the concept of epiplexity operational through a differentiable
spectral approximation. Their observer combines a fixed nonlinear feature map
with a fitted linear readout. This readout then predicts a trainable encoder's
representation, and the encoder is optimized to increase the spectral score of
the fitted readout. The observer thus participates in learning rather than
merely evaluating its outcome: its construction shapes the signal through which
the representation is trained. This provides a concrete setting in which to ask
how the organization of the observer influences the organization of what is
learned. Prior work shows that the choice of views affects which task
information is retained in contrastive learning~\citep{tian2020views}.

Together, these developments provide a way to measure learnable structure and
use its approximation as a training objective. However, we find that a
criterion for extracting structure is not necessarily a criterion for making
that structure useful. Our theoretical formulation examines the distinction
through a counterexample in which representations share the same covariance
spectrum while differing in task utility.

The issue is not that spectral objectives cannot reward relationships, but that
spectral diversity alone does not specify which relationships should be
learned. We therefore turn to the observer as a place to introduce a more
explicit objective on how the whole is explained. Consider a tree. A learner
may capture recurring leaf shapes, bark textures, and branch forms in reusable
descriptions. However, a catalog of these parts does not explain how a
particular tree is put together. To trace a path from a leaf to the trunk, the
learner must represent which branches connect to which. Knowing the pieces is
not knowing the whole. So, to summarize, we can decompose intelligence into
three parts. The first is the ability to find patterns. The second is to reduce
these patterns into efficient, reusable parts. And the third, our proposition:
composition.

In this paper, we make three contributions:
\begin{enumerate}
  \item We analyze the spectral approximation to epiplexity and show that,
  under a fixed-trace constraint, the objective favors a more uniform
  distribution of spectral mass. We complement this analysis with a
  counterexample to show that the uniform distribution of spectral mass does
  not necessarily contribute to downstream utility.
  \item We train four different models on the full ImageNet-1K training set
  and evaluate performance on CIFAR-100, DTD, EuroSAT, CLEVR, and SpatialSense.
  We find that input pressure is the leading reason for improved performance
  over any architectural change.
  \item Through our experimentation, we find that three properties of a
  representation can diverge: how broadly it distributes variance, how well an
  observer can predict it, and how effectively a downstream learner can use it.
\end{enumerate}

\section{Background}
\label{sec:background}

\subsection{Epiplexity}
\label{sec:epiplexity}

\citet{finzi2026epiplexity} define the time-bounded minimum description length
of a data object $D$ as
\begin{equation}
  \mathrm{MDL}_T(D) = \min_{P \in \mathcal{P}_T}
  \big\{ \ell(P) + \E_D[-\log_2 p_P(D)] \big\},
\end{equation}
where $\mathcal{P}_T$ contains programs satisfying the time bound $T$,
$\ell(P)$ is the length of $P$ in bits, and $\E_D$ is the expectation over the
data object $D$, treated as a random variable. The two terms balance the bits
needed to describe the model against those needed to encode the data using it.
Epiplexity is $S_T(D) = \ell(P^\star_T)$, the description length in bits of the
program $P^\star_T$ that minimizes this time-bounded objective.

For random data objects $\mathcal{U}, \mathcal{V}$, conditional epiplexity is
$S_T(\mathcal{V} \mid \mathcal{U}) = \ell(P^\star_{\mathcal{V} \mid \mathcal{U}})$,
where
\begin{equation}
  P^\star_{\mathcal{V} \mid \mathcal{U}} \in
  \argmin_{P \in \mathcal{P}^{\mathcal{U}}_T}
  \big\{ \ell(P) + \E_{(\mathcal{U},\mathcal{V})}
  [-\log_2 p_P(\mathcal{V} \mid \mathcal{U})] \big\}
\end{equation}
and $\mathcal{P}^{\mathcal{U}}_T$ contains admissible conditional probabilistic
programs satisfying the time bound $T$. It is the description length in bits
of the program minimizing the expected combined cost of describing the model
and encoding $\mathcal{V}$ given $\mathcal{U}$. Epiplexity can be estimated
using prequential and requential coding, but neither naturally provides the
cheap, differentiable, per-batch signal required for direct gradient-based
representation learning. This motivates a closed-form surrogate that can be
optimized during training.

\subsection{The spectral observer}
\label{sec:spectral-observer}

\citet{zhang2026learnable} make the conditional epiplexity quantity tractable
using a fixed observer and a closed-form readout. We describe it below. For a
batch $X = (x_1, \dots, x_N)$,
$z_i = E_\theta(x_i) \in \R^D,\ \ \|z_i\|_2 = 1,\ \ h_i = \phi(x_i) \in \R^m$,
where $E_\theta$ is the trainable, unit-normalized encoder and $\phi$ is a
fixed random nonlinear feature map. Originally, the scored targets are a
system's outputs; here they are the learned representations $z_i$. Collect
these into matrices $Z_\theta \in \R^{N \times D}$ and $H \in \R^{N \times m}$.
These observer features and targets are then standardized, giving
$\widetilde{H}$ and $\widetilde{Z}_\theta$ (\appref{app:preprocessing}).

The observer fits a linear readout by ridge regression~\citep{hoerl1970ridge},
using a quadratic penalty of the form studied in Tikhonov
regularization~\citep{tikhonov1977solutions}:
\begin{equation}
  W_\lambda
  = \argmin_{W \in \R^{m \times D}}
    \Big\{ \|\widetilde{Z}_\theta - \widetilde{H} W\|_F^2
    + \lambda \|W\|_F^2 \Big\}
  = \big( \widetilde{H}^\top \widetilde{H} + \lambda I_m \big)^{-1}
    \widetilde{H}^\top \widetilde{Z}_\theta,
  \qquad \lambda > 0,
  \label{eq:ridge-readout}
\end{equation}
so that the observer's predictions in the preprocessed target coordinates are
$\widehat{\widetilde{Z}}_\theta = \widetilde{H} W_\lambda$. The spectral
approximation to epiplexity is
\begin{equation}
  S^\phi(Z_\theta \mid X)
  = \frac{1}{2} \sum_j \log_2 \big( 1 + \eta\, s_j(W_\lambda)^2 \big),
  \qquad \eta > 0,
  \label{eq:spectral-epiplexity}
\end{equation}
where $s_j(W_\lambda)$ denotes a singular value of the fitted readout and
$\eta$ is the spectral resolution parameter. The encoder is trained by
\begin{equation}
  \min_\theta \; -S^\phi(Z_\theta \mid X).
  \label{eq:spectral-objective}
\end{equation}

Clearly, they have proposed an elegant approach to make epiplexity usable.
There is now a way to extract patterns and their compressed forms from data.
However, we still need to answer what these spectral scores are rewarding in a
representation.

\section{What the Spectral Objective Rewards}
\label{sec:analysis}

\subsection{Spectral preference under a fixed-trace condition}
\label{sec:fixed-trace}

We show that the spectral objective tends to reward a flatter spectrum. The
log-determinant score depends both on the total magnitude of the readout and on
how that magnitude is distributed across spectral directions. To isolate the
second effect, consider a fixed total readout energy: $\operatorname{tr}(W_\lambda^\top W_\lambda) = \|W_\lambda\|_F^2 = \tau$.
The following proposition characterizes the resulting spectral preference.
The bound follows from the concavity of the logarithm and Jensen's
inequality~\citep{boyd2004convex}.

\begin{proposition}[Fixed-trace spectral preference]
\label{prop:fixed-trace}
Let $A \in \R^{D \times D}$ be positive semidefinite, with $\operatorname{tr}(A) = \tau > 0,\ \ \operatorname{rank}(A) \le r,\ \ 1 \le r \le D$.
For $\eta > 0$,
\begin{equation}
  \frac{1}{2} \log_2 \det(I_D + \eta A)
  \le \frac{r}{2} \log_2 \Big( 1 + \frac{\eta \tau}{r} \Big).
  \label{eq:fixed-trace-bound}
\end{equation}
Equality holds if and only if $A$ has exactly $r$ nonzero eigenvalues, each
equal to $\tau / r$. \proofref{app:proof-fixed-trace}
\end{proposition}

Applying Proposition~\ref{prop:fixed-trace} to $A = W_\lambda^\top W_\lambda$
shows that, \emph{at fixed readout energy, the score favors distributing
spectral mass evenly across the available directions}. Moreover, the bound is
increasing in $r$, so at fixed energy the score also prefers using more
directions. Since $W_\lambda \in \R^{m \times D}$ has rank at most
$\min(m, D)$, when $m \ge D$ all $D$ readout directions are feasible and the
bound becomes
\begin{equation}
  S^\phi(Z_\theta \mid X)
  \le \frac{D}{2} \log_2 \Big( 1 + \frac{\eta \tau}{D} \Big).
  \label{eq:fixed-trace-full}
\end{equation}

This motivates us to answer the next question: whether a uniform distribution
of spectral mass helps a representation's downstream task utility.

\subsection{Equal spectra do not imply equal task utility}
\label{sec:counterexample}

We propose a counterexample, considering two different representations of the
same input, with the same covariance spectra, and show that they can perform
differently on the same downstream task (Figure~\ref{fig:same-spectra}).

\begin{figure}[t]
\centering
\includegraphics[width=\linewidth]{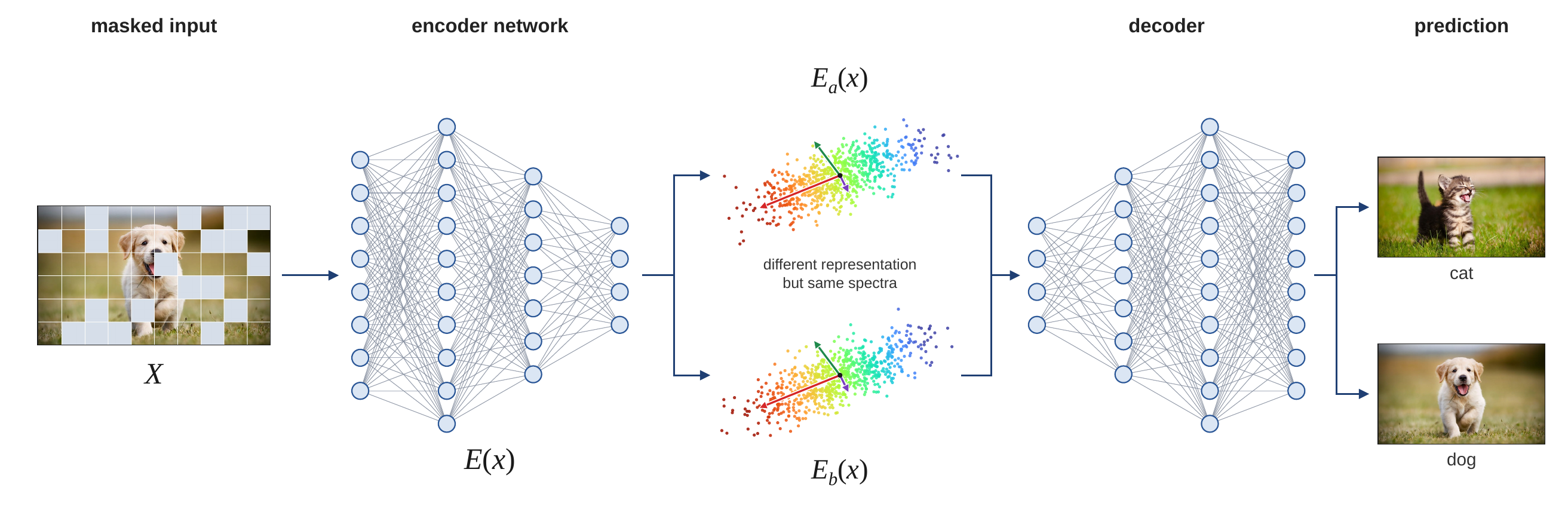}
\caption{Identical spectra need not give the same prediction. Two representations
$E_a(x)$ and $E_b(x)$ of the same masked input share a spectrum but lead the
decoder to different outputs. A conceptual illustration, not trained models.}
\label{fig:same-spectra}
\end{figure}

\begin{proposition}
\label{prop:counterexample}
There exist two unit-norm representations and a fixed observer such that their
representation covariances are identical and isotropic, their spectral scores
are identical, but one representation permits perfect linear classification
while the other permits only chance-level classification.
\proofref{app:proof-counterexample}
\end{proposition}

This counterexample ultimately shows us that the spectrum by itself does not
identify which features of the input occupy the representation's directions.
So, under the spectral objective $S^\phi$, we find that there
is no signal that forces the representation to be useful for downstream tasks.
This motivates us to think about what the observer rewards, and in
Section~\ref{sec:method} we elucidate our core contribution of a Compositional
Objective.

\section{The Compositional Objective}
\label{sec:method}

\subsection{Formalization of the Compositional Objective}
\label{sec:comp-observer}

In this section we introduce the notion of a compositional objective. We
construct this objective by providing inputs that are decomposed into multiple
parts. These parts form three kinds of tokens (Figure~\ref{fig:observer-tokens}):

\begin{figure}[t]
\centering
\includegraphics[width=\linewidth]{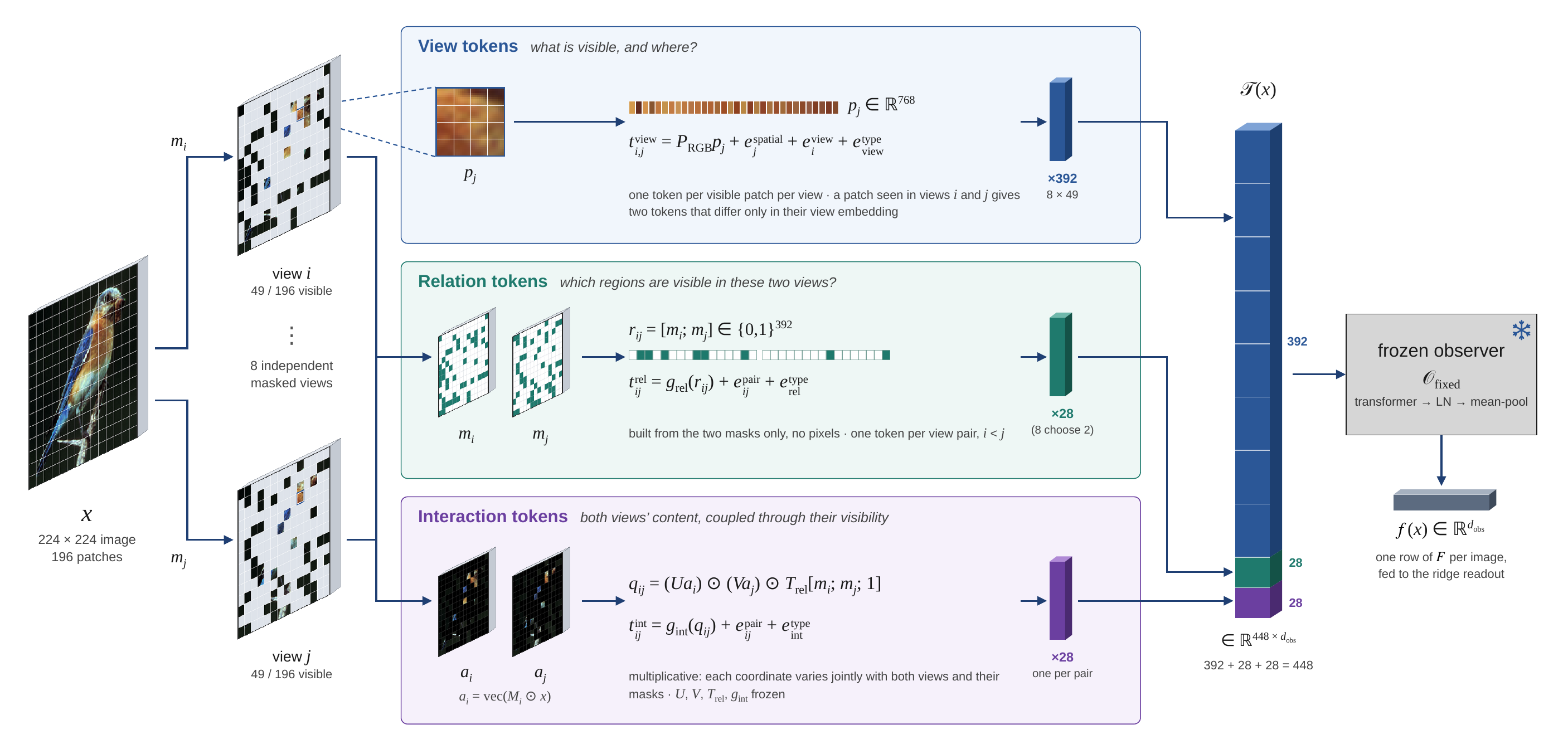}
\caption{The compositional observer's input tokens. Eight masked views of a
$196$-patch image give $392$ view tokens, $28$ relation tokens from the
visibility masks of each view pair, and $28$ interaction tokens that couple both
views' content with their visibility. A frozen transformer maps the $448$ tokens
to one feature vector $f(x)$.}
\label{fig:observer-tokens}
\end{figure}

\begin{enumerate}
  \item \emph{View tokens}: These are inputs that are masked differently. Every
  different mask makes a different view.
  \item \emph{Relation tokens}: We posit that the observer needs to have
  information to understand how each view relates to each other view.
  \item \emph{Interaction tokens}: For every view token pair and its
  corresponding relation token, we construct a fixed projection of the
  multiplicative interaction $(U u_i) \odot (V u_j) \odot (T\, [r_{ij}; 1])$,
  where $u_i, u_j$ are the two views, $r_{ij}$ is their relation, and $U$, $V$,
  and $T$ are fixed Gaussian matrices.
\end{enumerate}
We study two main ablations. The first, $\Sphicomp$, keeps the original
epiplexity objective $S^\phi$ of Section~\ref{sec:spectral-observer} and passes
its observer only the view tokens. The second combines all three kinds of
tokens into the observer's feature map, $\phi_{\mathrm{comp}}$, which our
compositional objective $\Scomp$ uses. Rather than retaining the spectral
objective of $S^\phi$ and changing only its input to $\phi_{\mathrm{comp}}$,
$\Scomp$ measures how much representation variance the observer actually
explains. We try to distinguish between using the decomposed input as the only
signal for composition, versus altering the observer itself to fit the input
uniquely.

\paragraph{Centering and the shared readout.}
Our objective uses its own preprocessing and readout, distinct from those of
Section~\ref{sec:spectral-observer}. Let
$f_i = \phi_{\mathrm{comp}}(x_i; \mathcal{M}_i) \in \R^m$ denote the
compositional observer features. We stack the centered features
$f_i - \bar{f}$, where $\bar{f}$ is their batch mean, into
$F \in \R^{N \times m}$, and the centered encoder outputs into
$\bar{Z}_\theta = Z_\theta - \mathbf{1}_N \mu_Z^\top \in \R^{N \times D}$, where
$\mu_Z$ is the batch mean of the encoder outputs. Neither is standardized. We
define the covariances
\begin{equation}
  C_F = \frac{1}{N} F^\top F, \qquad
  C_Z = \frac{1}{N} \bar{Z}_\theta^\top \bar{Z}_\theta, \qquad
  C_{FZ} = \frac{1}{N} F^\top \bar{Z}_\theta.
  \label{eq:comp-covariances}
\end{equation}
Centering makes the target mean the reference predictor, so $\Scomp$ measures
prediction of variation across examples rather than of average activations.
The readout solves a mean-error ridge problem,
\begin{equation}
  B
  = \argmin_{B' \in \R^{m \times D}}
    \Big\{ \frac{1}{N} \|\bar{Z}_\theta - F B'\|_F^2
    + \gamma \|B'\|_F^2 \Big\}
  = (C_F + \gamma I_m)^{-1} C_{FZ}, \qquad \gamma > 0.
  \label{eq:comp-readout}
\end{equation}
Relative to the readout $W_\lambda$ of \eqref{eq:ridge-readout}, $B$ uses
centering without standardization, $u_Z = 1$, and $\gamma = \lambda / N$. The
observer remains fixed, while the readout is recomputed from the current
encoder outputs.

\subsection{\texorpdfstring{$\Scomp$}{S\_comp}: predictive covariance gain}
\label{sec:gcomp}

A spectrally rich readout need not explain a substantial amount of
representation variation. $\Scomp$ addresses this distinction by comparing the
target covariance with the regularized covariance remaining after prediction
(Figure~\ref{fig:compositional-objective}).

\begin{figure}[t]
\centering
\includegraphics[width=\linewidth]{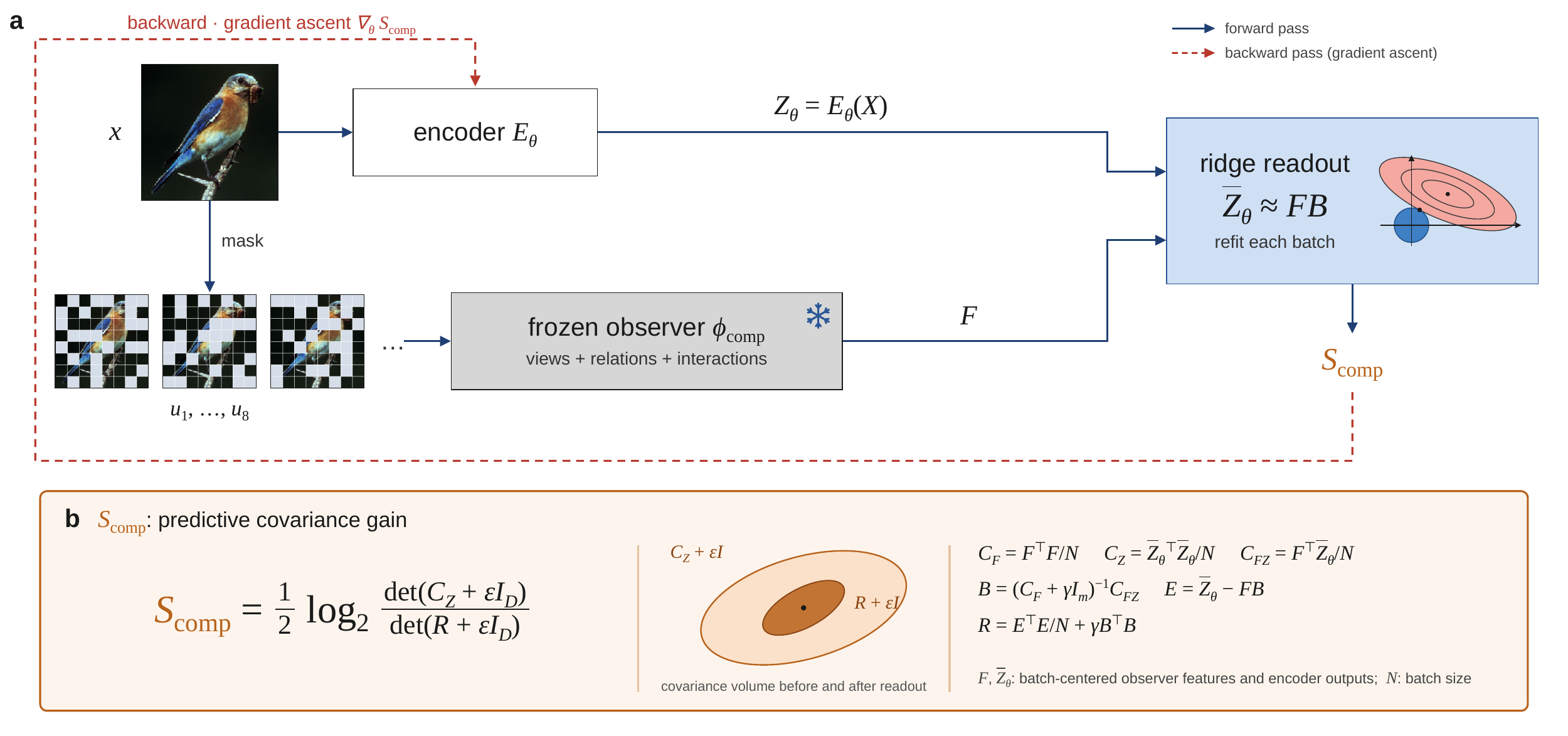}
\caption{The compositional objective. (a) The encoder $E_\theta$ sees the whole
image, the frozen observer $\phi_{\mathrm{comp}}$ builds features $F$ from
masked views, and a ridge readout $B$, refitted each batch, predicts the centered
encoder outputs; the encoder ascends $\Scomp$. (b) $\Scomp$ compares the target
covariance $C_Z + \epsilon I_D$ with the regularized remaining cost
$R + \epsilon I_D$.}
\label{fig:compositional-objective}
\end{figure}

Define $E = \bar{Z}_\theta - F B$
and
\begin{equation}
  R = \frac{1}{N} E^\top E + \gamma B^\top B
    = C_Z - C_{FZ}^\top (C_F + \gamma I_m)^{-1} C_{FZ},
  \label{eq:gcomp-residual}
\end{equation}
where $R$ is the unexplained variance. Further, the objective is
\begin{equation}
  \Scomp(Z_\theta \mid X)
  = \frac{1}{2} \log_2 \frac{\det(C_Z + \epsilon I_D)}{\det(R + \epsilon I_D)},
  \qquad \epsilon > 0.
  \label{eq:gcomp}
\end{equation}
Here, $\epsilon$ provides a positive covariance floor. The matrix $R$ includes
the ridge penalty and is therefore \emph{not simply the covariance of
prediction errors}. The score measures reduction in regularized covariance
volume relative to the constant-mean predictor. As $\gamma, \epsilon \to 0$,
$\Scomp$ approaches $\tfrac{1}{2} \log_2 \det C_Z / \det(C_Z - C_{FZ}^\top
C_F^{-1} C_{FZ})$, the mutual information in bits between the observer
features and the representation when they are jointly Gaussian, with
nonsingular marginal covariances and positive-definite conditional
covariance~\citep{bach2002kernel}. The finite-regularization score $\Scomp$ is
a regularized covariance-volume measure.

For positive $\gamma$ and $\epsilon$,
\begin{equation}
  \Scomp \ge 0, \qquad \Scomp = 0 \iff C_{FZ} = 0.
  \label{eq:gcomp-properties}
\end{equation}
\proofref{app:proof-gcomp} Thus, the score
vanishes when the centered observer features have no linear cross-covariance
with the representation.

Unlike $S^\phi$, which \emph{scores the spectral content of the readout},
$\Scomp$ \emph{scores the reduction in regularized target covariance the
readout achieves}, so the two need not favor the same representations. The
encoder minimizes
\begin{equation}
  \mathcal{L}(\theta) = -\Scomp(Z_\theta \mid X),
  \label{eq:comp-losses}
\end{equation}
and compositionality enters only through the observer's features.

\section{Experiments}
\label{sec:experiments}

We test the three contributions on full ImageNet. Section~\ref{sec:spreading}
tests the spectral preferences of Section~\ref{sec:analysis} and
\appref{app:analysis}, Section~\ref{sec:observer-input} tests whether
performance follows the observer's input, and
Section~\ref{sec:three-properties} tests whether spread, predictability, and
utility diverge.

\subsection{Setup}
\label{sec:setup}

\paragraph{Pretraining.} We pretrain on the full
ImageNet-1K~\citep{russakovsky2015imagenet} training set ($1{,}281{,}167$
images) with a ViT encoder~\citep{dosovitskiy2021vit} whose $192$-dimensional output is
unit-normalized. Each method is trained for ten epochs ($33{,}360$ updates at
batch size $384$) with two pretraining seeds, without class labels. All results
evaluate the final checkpoint with the encoder frozen. Training details and
hardware are in \appref{app:experiment-details}.

\paragraph{Objectives and notation.} $S^\phi$ is the original spectral
objective of Section~\ref{sec:spectral-observer}, with a whole-image observer
and its reference preprocessing conventions. $\Sphicomp$ applies the same
spectral score to an observer that sees only the eight masked views of
Section~\ref{sec:comp-observer}, without relation or interaction tokens.
$\Scomp$ is the objective of
Section~\ref{sec:gcomp}, using the views, relations, and interactions of
Section~\ref{sec:comp-observer}. In all three methods, the trainable encoder
receives the whole image.

\paragraph{Evaluation.} We report ImageNet linear-probe top-1 and top-5
accuracy and $5$-nearest-neighbor accuracy, a separate low-label ImageNet
evaluation with $20$ labels per class, and transfer to
CIFAR-100~\citep{krizhevsky2009learning}, DTD~\citep{cimpoi2014describing},
and EuroSAT~\citep{helber2019eurosat}. For relational evaluation, we report
SpatialSense~\citep{yang2019spatialsense} and a relational-question subset of
CLEVR~\citep{johnson2017clevr}, using readouts on encoder features and
metadata.

\begin{table}[t]
\caption{Frozen-encoder accuracy (\%) after ten epochs on ImageNet-1K, averaged
over two seeds. Bold marks the best of the three methods; per-seed values are in
\appref{app:experiment-details}.}
\label{tab:stage2-accuracy}
\centering
\small
\begin{tabular}{lccc}
\toprule
Evaluation & $S^\phi$ & $\Sphicomp$ & $\Scomp$ \\
\midrule
ImageNet top-1 & 3.19 & \textbf{7.02} & 3.30 \\
ImageNet top-5 & 9.57 & \textbf{17.21} & 9.64 \\
ImageNet 5-NN & 1.62 & \textbf{2.67} & 1.48 \\
Low-label ImageNet & 1.94 & \textbf{2.23} & 2.07 \\
CIFAR-100 & 8.32 & \textbf{9.68} & 6.94 \\
DTD & 11.22 & 11.20 & \textbf{11.33} \\
EuroSAT & 63.50 & \textbf{68.02} & 65.69 \\
\midrule
SpatialSense & \textbf{52.84} & 52.75 & 52.32 \\
CLEVR relational subset & \textbf{44.49} & 44.44 & 44.27 \\
\bottomrule
\end{tabular}
\end{table}

\begin{table}[t]
\caption{Representation diagnostics and counted training compute, averaged over
two seeds. Geometry uses the low-label ImageNet test features; held-out
reduction uses the common compositional observer and is unavailable for
$S^\phi$.}
\label{tab:stage2-diagnostics}
\centering
\small
\begin{tabular}{lccc}
\toprule
Measurement & $S^\phi$ & $\Sphicomp$ & $\Scomp$ \\
\midrule
Effective rank (of 192) & 39.27 & 175.23 & 27.07 \\
Isotropy & 0.194 & 0.869 & 0.112 \\
Held-out error reduction (\%) & -- & 18.56 & 83.51 \\
Counted TFLOPs per update & 3.826 & 5.055 & 5.810 \\
Counted PFLOPs per run & 127.62 & 168.64 & 193.82 \\
\bottomrule
\end{tabular}
\end{table}

\subsection{The spectral objective spreads the representation}
\label{sec:spreading}

Proposition~\ref{prop:fixed-trace} predicts that the spectral score rewards a
flatter spectrum. With the masked-view observer, $\Sphicomp$ produces a nearly
isotropic representation, with effective rank $175.23$ of $192$ and isotropy
$0.869$ (Table~\ref{tab:stage2-diagnostics}, Figure~\ref{fig:stage2-spectra}).
$\Scomp$, whose score is convex
in its improvement eigenvalues (\appref{app:proof-gcomp}), instead concentrates
its spectrum, with effective rank $27.07$ and isotropy $0.112$. Each geometry
follows the preference of its score. With the whole-image observer, the same
spectral score, $S^\phi$, reaches isotropy $0.194$: the spectrum the score can attain
depends on what the observer sees. The most spread representation is not the
most useful on every task: it trails $\Scomp$ on DTD and $S^\phi$ on both
relational evaluations (Table~\ref{tab:stage2-accuracy}), as
Proposition~\ref{prop:counterexample} allows.

\subsection{Observer input drives performance}
\label{sec:observer-input}

Changing what the spectral observer sees, from the whole image ($S^\phi$) to
masked views ($\Sphicomp$), produces the largest gains in Table~\ref{tab:stage2-accuracy}
(Figure~\ref{fig:stage2-input}). ImageNet top-1
accuracy rises from $3.19\%$ to $7.02\%$, top-5 from $9.57\%$ to $17.21\%$,
low-label accuracy from $1.94\%$ to $2.23\%$, CIFAR-100 from $8.32\%$ to
$9.68\%$, and EuroSAT from $63.50\%$ to $68.02\%$; each difference is positive
in both seeds. DTD is unchanged ($11.22\%$ and $11.20\%$). The architectural
change in $\Scomp$, which adds relation and interaction tokens and replaces the
spectral score, does not add to these gains: $\Scomp$ is below $\Sphicomp$ on
every classification evaluation except DTD, where it is highest
($11.33\%$). On SpatialSense and the CLEVR relational subset, the three methods
are within $0.6$ points of one another, and their ordering changes across
seeds.

\begin{figure}[H]
\centering
\begin{minipage}[t]{0.485\linewidth}
\centering
\includegraphics[width=\linewidth]{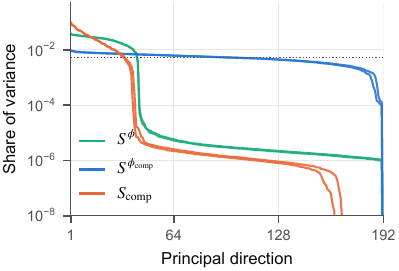}
\caption{Variance share per principal direction of the low-label ImageNet test
features, one line per seed. Dotted: isotropic share $1/192$.}
\label{fig:stage2-spectra}
\end{minipage}\hfill
\begin{minipage}[t]{0.485\linewidth}
\centering
\includegraphics[width=\linewidth]{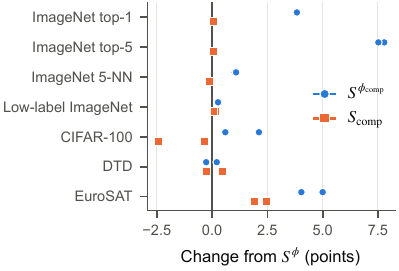}
\caption{Accuracy change from $S^\phi$ of the same seed on each evaluation, one mark
per seed.}
\label{fig:stage2-input}
\end{minipage}
\end{figure}

\subsection{Spread, predictability, and utility diverge}
\label{sec:three-properties}

Held-out predictability gives a third ordering. Under a common frozen
observer, a ridge readout reduces held-out prediction error by $83.51\%$ for
$\Scomp$, compared with $18.56\%$ for $\Sphicomp$
(Table~\ref{tab:stage2-diagnostics}; Figure~\ref{fig:stage2-dissociation}), yet $\Scomp$ is less accurate on most
evaluations. $\Sphicomp$ is the most spread and the most accurate on
six of the seven classification evaluations, but it is the less predictable of
the two; $\Scomp$ is the most predictable and the least spread; and $S^\phi$
is best on the two relational readouts. Held-out error
reduction is an evaluation metric, not the training score. Spread,
predictability, and utility therefore order the methods differently.

\begin{figure}[H]
\centering
\includegraphics[width=0.82\linewidth]{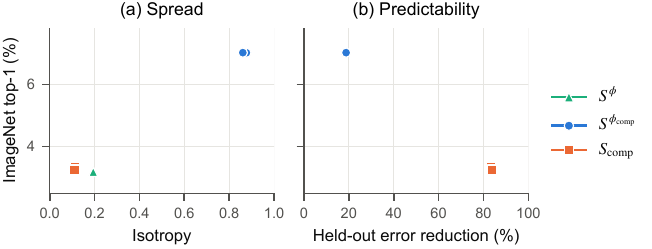}
\caption{ImageNet top-1 accuracy against (a) isotropy and (b) held-out error
reduction, one mark per method and seed.}
\label{fig:stage2-dissociation}
\end{figure}

\subsection{Limitations}
\label{sec:limitations}

Two seeds permit replication and paired comparisons, but not strong
significance claims. The change from $S^\phi$ to $\Sphicomp$ also changes image preprocessing and score calibration
(\appref{app:experiment-details}), so its gains cannot be attributed to the
observer's input alone, and the comparison between $\Sphicomp$ and $\Scomp$ does not separate the relation and interaction tokens from the score.
The relational readouts combine encoder features with metadata and do not
isolate the encoder. Counted training compute differs across methods
(Table~\ref{tab:stage2-diagnostics}).

\section{Related Work}
\label{sec:related-work}

Our work builds on the minimum description length
principle~\citep{rissanen1978modeling,grunwald2007mdl} and computationally
bounded notions of information~\citep{xu2020usable,finzi2026epiplexity}, and on
\citet{zhang2026learnable}, whose learnable-novelty estimator resolves the
failure modes of curiosity and surprise-minimization
objectives~\citep{schmidhuber2010formal,burda2019largescale,friston2010free}.
These frameworks measure how much structure a bounded observer can extract,
but not which structure a representation should retain. Spectral and
redundancy-reduction
objectives~\citep{yu2020mcr2,zbontar2021barlow,bardes2022vicreg} and spectral
proxies for representation
quality~\citep{garrido2023rankme,agrawal2022alphareq} reward spread; our
counterexample shows that spread alone cannot determine utility, paralleling
the finding that maximizing mutual information does not by itself yield
useful representations~\citep{tschannen2020mi}. Dimensional collapse can also
occur in contrastive representation learning~\citep{jing2022collapse}. Masked
and context-prediction
methods~\citep{he2022mae,assran2023ijepa} learn to predict from visible parts,
and part--whole architectures build composition into the
encoder~\citep{sabour2017capsules,hinton2023partwhole,locatello2020slot}; our
observer is instead fixed and random, in the tradition of reservoir
computing~\citep{jaeger2004harnessing,maass2002realtime}. Fixed random
features followed by a fitted linear model also appear in random-feature
methods~\citep{rahimi2007random}. Our main
contribution is to analyze what spectral epiplexity objectives reward: we show
that spectral spread alone cannot determine utility, and we find empirically
that spread, predictability, and utility can diverge.

\section{Conclusion}
\label{sec:conclusion}

We started with a simple question: when the spectral approximation to
epiplexity is used to train a representation, what is it actually rewarding?
In theory, the answer is spread. At fixed readout energy, the score prefers to
distribute spectral mass evenly across as many directions as it can, yet two
representations with identical spectra can be useful in entirely different
ways. Spread alone does not tell us what a representation is good for.

Our experiments on full ImageNet-1K bear this out. The biggest lever we found
was not a new architecture but what the observer gets to see: switching the
spectral observer from the whole image to masked views more than doubled
linear-probe top-1 accuracy. Adding relation and interaction tokens and
switching to our predictive score, $\Scomp$, did not add to that gain, even
though it made the representation far more predictable. Spread,
predictability, and utility turned out to be three different properties, and
they rank the same methods in three different orders.

To return to the tree from the introduction: recognizing the leaves is not the
same as knowing how they connect, and an objective that only measures how much
structure is present cannot tell the two apart. Building objectives that reward
the connections themselves, and testing them across more seeds, at larger
scale, and on tasks that directly probe compositional generalization, is where
we want to go next.

\subsection*{AI use statement}

In this work, we used generative AI tools to help develop theoretical models or
conceptual frameworks, formulate mathematical claims, provide critical
ingredients for proving mathematical claims, assist in the writing of proofs,
propose or refine hypotheses, design or provide feedback on research
methodology or experiments, and implement methods. We have not used generative
AI tools to generate synthetic datasets, clean or reformat datasets, support
qualitative and thematic data analysis, or interpret results; these uses are
not applicable to this work. Additionally, we used generative AI tools to
create or modify scientific figures or images, suggest experimental parameters,
create or edit software code, create artifacts, draft parts of a research
paper, and summarize or analyze existing literature. We have reviewed all
AI-assisted work. We have manually verified the proofs, and we have thoroughly
vetted the code, results, and claims made in the paper. We take responsibility
for the final content of this work, including text, claims, or artifacts
produced with the aid of generative AI.

\bibliography{references}
\bibliographystyle{iclr2027_conference}

\clearpage
\appendix
\counterwithin{figure}{section}
\counterwithin{table}{section}
\counterwithin{equation}{section}
\section{Spectral Observer Preprocessing}
\label{app:preprocessing}

Following \citet{zhang2026learnable}, the observer features and targets are
standardized as
\begin{equation}
  \widetilde{H}_{ic} = \frac{H_{ic} - \mu_c}{\widehat{\sigma}_c \sqrt{m}},
  \qquad
  \widetilde{Z}_\theta = \frac{Z_\theta - \mathbf{1}_N \mu_Z^\top}{u_Z},
  \label{eq:preprocessing}
\end{equation}
where $\mu_c$ and $\widehat{\sigma}_c$ are the empirical mean and standard
deviation of observer feature $c$, $\mu_Z$ is the empirical mean encoder
output, and $\mathbf{1}_N$ is the vector of $N$ ones. The target scale
$u_Z > 0$ is fixed in advance, not estimated from the targets' empirical
standard deviation.

\section{Proofs}
\label{app:proofs}

\subsection{Proof of Proposition~\ref*{prop:fixed-trace}}
\label{app:proof-fixed-trace}

Let $\rho_1, \dots, \rho_r$ denote the potentially nonzero eigenvalues of $A$,
padding the list with zeros when $\operatorname{rank}(A) < r$. Then
\begin{equation*}
  \rho_j \ge 0, \qquad \sum_{j=1}^r \rho_j = \tau.
\end{equation*}
The remaining $D - r$ eigenvalues contribute zero to the log-determinant, so
\begin{equation*}
  \frac{1}{2} \log_2 \det(I_D + \eta A)
  = \frac{1}{2} \sum_{j=1}^r \log_2 (1 + \eta \rho_j).
\end{equation*}
The function $g(u) = \log_2(1 + \eta u)$ is strictly concave for $u \ge 0$,
since
\begin{equation*}
  g''(u) = -\frac{\eta^2}{\ln(2) (1 + \eta u)^2} < 0.
\end{equation*}
Jensen's inequality therefore gives
\begin{equation*}
  \frac{1}{r} \sum_{j=1}^r \log_2 (1 + \eta \rho_j)
  \le \log_2 \Big( 1 + \frac{\eta}{r} \sum_{j=1}^r \rho_j \Big)
  = \log_2 \Big( 1 + \frac{\eta \tau}{r} \Big).
\end{equation*}
Multiplying by $r/2$ proves the bound. Strict concavity implies equality only
when $\rho_1 = \dots = \rho_r = \tau / r$. Because $\tau > 0$, these
eigenvalues are all positive, so equality also requires
$\operatorname{rank}(A) = r$.

Proposition~\ref*{prop:fixed-trace} is a conditional statement: it does not
establish that training holds the readout trace fixed or reaches the bound,
and it concerns the readout rather than the representation.

\subsection{Proof of Proposition~\ref*{prop:counterexample}}
\label{app:proof-counterexample}

Let $t, a, b$ be independent random variables, each uniformly distributed on
$\{-1, +1\}$. Define the input and downstream label by $x = (t, a, b),\ \ y = t$.
Here, $t$ determines the label, while $a$ and $b$ are irrelevant to this
particular classification task. Consider the representations
\begin{equation}
  z_A(x) = \frac{1}{\sqrt{2}} \begin{pmatrix} t \\ a \end{pmatrix}, \qquad
  z_B(x) = \frac{1}{\sqrt{2}} \begin{pmatrix} a \\ b \end{pmatrix}.
  \label{eq:ce-representations}
\end{equation}
Both representations have unit norm and zero mean. Independence and the unit
variance of each random sign give $\operatorname{Cov}(z_A) = \operatorname{Cov}(z_B) = \frac{1}{2} I_2$.
Both therefore have full rank and identical, perfectly uniform covariance
spectra. Now choose a fixed observer, here the identity feature map rather than
a random one, $\phi(x) = (t, a, b)^\top$.

Take a balanced batch $X$ of $N$ inputs containing all eight possible triples
equally often, and let $Z_A, Z_B \in \R^{N \times 2}$ stack the two
representations. Every observer feature and every representation coordinate
then has empirical mean zero, and each observer feature has unit empirical
standard deviation, so the preprocessing of \appref{app:preprocessing}
gives $\widetilde{H} = H / \sqrt{3}$, $\widetilde{Z}_A = Z_A / u_Z$, and
$\widetilde{Z}_B = Z_B / u_Z$. Independence of the signs gives
\begin{equation*}
  \widetilde{H}^\top \widetilde{H} = \frac{N}{3} I_3, \qquad
  \widetilde{H}^\top \widetilde{Z}_A = \frac{N}{\sqrt{6}\, u_Z} J_A, \qquad
  \widetilde{H}^\top \widetilde{Z}_B = \frac{N}{\sqrt{6}\, u_Z} J_B,
\end{equation*}
where
\begin{equation}
  J_A = \begin{pmatrix} 1 & 0 \\ 0 & 1 \\ 0 & 0 \end{pmatrix}, \qquad
  J_B = \begin{pmatrix} 0 & 0 \\ 1 & 0 \\ 0 & 1 \end{pmatrix}
  \label{eq:ce-selections}
\end{equation}
select the observer coordinates $(t, a)$ and $(a, b)$. By
equation~\ref*{eq:ridge-readout}, the ridge readouts are
\begin{equation}
  W_A = c\, J_A, \qquad W_B = c\, J_B, \qquad
  c = \sqrt{\frac{3}{2}}\, \frac{N}{u_Z (N + 3\lambda)}.
  \label{eq:ce-readouts}
\end{equation}
Although the readouts select different observer coordinates, their Gram
matrices are identical, since $J_A^\top J_A = J_B^\top J_B = I_2$:
\begin{equation}
  W_A^\top W_A = W_B^\top W_B = c^2 I_2.
  \label{eq:ce-gram}
\end{equation}

Their spectral scores are consequently equal:
\begin{equation}
  S^\phi(Z_A \mid X) = S^\phi(Z_B \mid X) = \log_2 \big( 1 + \eta c^2 \big).
  \label{eq:ce-scores}
\end{equation}
Their classification properties are different. Since the first coordinate of
$z_A$ is $t / \sqrt{2}$, the linear decision rule $\widehat{y}_A(z_A) = \operatorname{sign}((z_A)_1)$
recovers $y = t$ exactly, whereas $z_B$ depends only on $a$ and $b$, which are
independent of $t$, so every classifier $h : \R^2 \to \{-1, +1\}$ is at chance:
\begin{equation*}
  \Pr[\widehat{y}_A(z_A) = y] = 1, \qquad \Pr[h(z_B) = y] = \tfrac{1}{2}.
\end{equation*}
This proves the claims.

If instead the observer is $\phi(x) = (t, a)$, then $\widetilde{H} = H / \sqrt{2}$,
$\widetilde{H}^\top \widetilde{H} = \tfrac{N}{2} I_2$,
$\widetilde{H}^\top \widetilde{Z}_A = \tfrac{N}{2 u_Z} I_2$, and
$\widetilde{H}^\top \widetilde{Z}_B = \tfrac{N}{2 u_Z}
\big(\begin{smallmatrix} 0 & 0 \\ 1 & 0 \end{smallmatrix}\big)$, since the
observer no longer receives $b$. The readouts are $W_A = c' I_2$ and
$W_B = c' \big(\begin{smallmatrix} 0 & 0 \\ 1 & 0 \end{smallmatrix}\big)$ with
$c' = N / (u_Z (N + 2\lambda))$, so $W_B^\top W_B$ has a single nonzero
eigenvalue $c'^2$ and
\begin{equation*}
  S^\phi(Z_B \mid X) = \tfrac{1}{2} \log_2(1 + \eta c'^2)
  < S^\phi(Z_A \mid X) = \log_2(1 + \eta c'^2).
\end{equation*}

\section{Mathematical Analysis of the Compositional Objective}
\label{app:analysis}

This appendix analyzes the compositional predictive-gain objective $\Scomp$.
We characterize its relationship to regularized prediction error, its
gradients, its sensitivity to unpredictable variation, its spectral allocation
preferences, and its dependence on representation scale. The results concern
the specified observer and the in-sample ridge fit; they do not establish
semantic compositionality or generalization to unseen combinations.

\subsection{Setup and objective definition}
\label{app:analysis-setup}

Let $F \in \R^{N \times m}$ contain the compositional observer features
constructed from image views and their relations, and let
$\bar{Z}_\theta \in \R^{N \times D}$ contain the corresponding whole-image
representations. Both matrices are column-centered across the $N$ examples.
For a frozen observer, $F$ is independent of the encoder parameters $\theta$,
conditional on the sampled images and views. As in
Section~\ref*{sec:comp-observer}, define
\begin{equation*}
  C_F = \frac{1}{N} F^\top F, \qquad
  C_Z = \frac{1}{N} \bar{Z}_\theta^\top \bar{Z}_\theta, \qquad
  C_{FZ} = \frac{1}{N} F^\top \bar{Z}_\theta.
\end{equation*}
The observer's readout is obtained from the inner ridge problem
\begin{equation*}
  B = \argmin_{B' \in \R^{m \times D}}
  \Big\{ \frac{1}{N} \|\bar{Z}_\theta - F B'\|_F^2 + \gamma \|B'\|_F^2 \Big\},
  \qquad \gamma > 0.
\end{equation*}
Writing $A = C_F + \gamma I_m$, the normal equations yield $B = A^{-1} C_{FZ}$.
Define
\begin{equation*}
  Q = C_{FZ}^\top A^{-1} C_{FZ}, \qquad R = C_Z - Q.
\end{equation*}
For a fixed covariance floor $\epsilon > 0$, the score is
\begin{equation}
  \Scomp = \frac{1}{2} \log_2
  \frac{\det(C_Z + \epsilon I_D)}{\det(R + \epsilon I_D)},
  \label{eq:analysis-score}
\end{equation}
and the encoder is trained by minimizing
$\mathcal{L}(\theta) = -\Scomp(\theta)$, as in equation~\ref*{eq:comp-losses}. The
ridge problem is minimized with respect to $B'$, whereas the outer loss is
minimized with respect to $\theta$. The dependence of the fitted readout and
the covariance matrices on the encoder representations is retained when
differentiating the outer loss.

\subsection{Interpretation as regularized predictive improvement}
\label{app:analysis-interpretation}

\begin{proposition}
\label{prop:analysis-residual}
At the ridge solution,
\begin{equation}
  R = \frac{1}{N} (\bar{Z}_\theta - F B)^\top (\bar{Z}_\theta - F B)
  + \gamma B^\top B.
  \label{eq:analysis-residual}
\end{equation}
\end{proposition}

\begin{proof}
Expanding the right-hand side gives
\begin{equation*}
  C_Z - C_{FZ}^\top B - B^\top C_{FZ} + B^\top A B.
\end{equation*}
The normal equations $A B = C_{FZ}$ imply
$C_{FZ}^\top B = B^\top C_{FZ} = B^\top A B = Q$. The expression therefore
equals $C_Z - Q = R$.
\end{proof}

Consequently, $Q \succeq 0$ and $R \succeq 0$, and
\begin{equation*}
  \operatorname{tr}(R) = \min_{B'}
  \Big\{ \frac{1}{N} \|\bar{Z}_\theta - F B'\|_F^2 + \gamma \|B'\|_F^2 \Big\},
  \qquad
  \operatorname{tr}(Q) = \operatorname{tr}(C_Z) - \operatorname{tr}(R).
\end{equation*}
Because the data are centered, $B' = 0$ corresponds to predicting the target
mean before centering. Thus, $\operatorname{tr}(C_Z)$ is the mean-only baseline
loss, $\operatorname{tr}(R)$ is the optimized ridge loss, and
$\operatorname{tr}(Q)$ is the improvement after accounting for regularization.

In particular, $R$ is \emph{residual covariance plus ridge weight cost}, not
residual covariance alone. Likewise, $Q$ represents regularized improvement
rather than simply the covariance of the fitted predictions. Since
$R \preceq C_Z$, the score is nonnegative. The positive floor $\epsilon$ makes
both log-determinants well defined even when the covariance matrices are
rank-deficient.

\subsection{Differential and training incentives}
\label{app:analysis-differential}

Write $C_\epsilon = C_Z + \epsilon I_D$ and $R_\epsilon = R + \epsilon I_D$.
The log-determinant differential gives
\begin{equation*}
  \mathrm{d}\Scomp = \frac{1}{2 \ln 2}
  \operatorname{tr}\big[ C_\epsilon^{-1}\, \mathrm{d}C_Z
  - R_\epsilon^{-1}\, \mathrm{d}R \big].
\end{equation*}
Using $R = C_Z - Q$, this becomes
\begin{equation}
  \mathrm{d}\Scomp = \frac{1}{2 \ln 2}
  \operatorname{tr}\big[ \big( C_\epsilon^{-1} - R_\epsilon^{-1} \big)
  \mathrm{d}C_Z + R_\epsilon^{-1}\, \mathrm{d}Q \big].
  \label{eq:analysis-differential}
\end{equation}
Since $R \preceq C_Z$, we have $C_\epsilon^{-1} - R_\epsilon^{-1} \preceq 0$.
Holding $C_Z$ fixed, a positive-semidefinite increase in $Q$ increases the
score. Holding $Q$ fixed, a positive-semidefinite increase in $C_Z$ cannot
increase the score. The objective therefore rewards predictive improvement
relative to the variation remaining to be explained, rather than target
variation alone.

These are conditional statements: during encoder training, $C_Z$ and $Q$
generally change together. Gradients of the minimized loss $\mathcal{L}$ have
the opposite sign to the score gradients above.

\subsection{Sensitivity to unpredictable variation}
\label{app:analysis-unpredictable}

\paragraph{A fixed-observer scalar example.}
Let $f$ and $e$ be scalar random variables satisfying
\begin{equation*}
  \E[f] = \E[e] = 0, \qquad \E[f^2] = \E[e^2] = 1, \qquad \E[f e] = 0.
\end{equation*}
Consider the target $z_t = f + t e$, where $t \in \R$ is a trainable scalar
parameter, not a random variable. The joint distribution of $f$ and $e$ is
held fixed, and the observer is restricted to predictions of the form
$\widehat{z} = b f$. This paragraph uses population moments; the same
identities hold empirically when the corresponding sample moments are exact.
No target standardization is applied initially. The ridge objective is
\begin{equation*}
  \E[(z_t - b f)^2] + \gamma b^2 = (1 - b)^2 + t^2 + \gamma b^2.
\end{equation*}
Hence $C_F = 1$, $C_{FZ} = 1$, $C_Z = 1 + t^2$, and $b = 1 / (1 + \gamma)$. It
follows that
\begin{equation*}
  Q = \frac{1}{1 + \gamma}, \qquad R = t^2 + \frac{\gamma}{1 + \gamma},
  \qquad
  \Scomp(t) = \frac{1}{2} \log_2
  \frac{1 + t^2 + \epsilon}{t^2 + \frac{\gamma}{1 + \gamma} + \epsilon}.
\end{equation*}
Differentiation yields
\begin{equation}
  \frac{\mathrm{d}\Scomp}{\mathrm{d}t}
  = -\frac{t}{(1 + \gamma) \ln 2\, (1 + t^2 + \epsilon)
  \big( t^2 + \frac{\gamma}{1 + \gamma} + \epsilon \big)}.
  \label{eq:analysis-scalar-derivative}
\end{equation}
The derivative is negative for $t > 0$ and positive for $t < 0$. Maximizing
the score therefore drives nonzero $t$ toward zero, suppressing the component
that this linear observer cannot predict. This does not imply that
observer-unpredictable information is necessarily irrelevant to other tasks.

\paragraph{Effect of target standardization.}
If the target is instead standardized to unit variance,
$\widetilde{z}_t = (f + t e) / \sqrt{1 + t^2}$, then $C_Z = 1$ and
$C_{FZ} = 1 / \sqrt{1 + t^2}$. The score becomes
\begin{equation*}
  \Scomp(t) = -\frac{1}{2} \log_2 \Big( 1 - \frac{1}{(1 + \epsilon)
  (1 + \gamma)(1 + t^2)} \Big).
\end{equation*}
This expression also decreases with $t^2$. Standardization changes the score
and its gradients, but in this example the preference for suppressing the
unpredictable component remains.

\paragraph{Appending an uncorrelated target coordinate.}
Consider appending a centered column $\xi \in \R^N$ satisfying
\begin{equation*}
  F^\top \xi = 0, \qquad \bar{Z}_\theta^\top \xi = 0, \qquad
  v_\xi = \frac{1}{N} \xi^\top \xi.
\end{equation*}
Keeping the observer features, the original target coordinates, and the
scaling unchanged, the augmented target
$\bar{Z}_\theta^{+} = [\bar{Z}_\theta, \xi]$ yields
\begin{equation*}
  C_Z^{+} = \operatorname{diag}(C_Z, v_\xi), \qquad
  Q^{+} = \operatorname{diag}(Q, 0), \qquad
  R^{+} = \operatorname{diag}(R, v_\xi).
\end{equation*}
The extra factor $v_\xi + \epsilon$ cancels between numerator and denominator,
so
\begin{equation}
  \Scomp(\bar{Z}_\theta^{+}) = \Scomp(\bar{Z}_\theta).
  \label{eq:analysis-append}
\end{equation}
An appended, uncorrelated coordinate therefore receives no additional reward,
but it is not actively penalized either. This differs from adding
unpredictable variation to an existing predictive coordinate. Independence
implies the corresponding population result, but finite samples need not have
exactly zero cross-covariances.

\subsection{Spectral representation and allocation preferences}
\label{app:proof-gcomp}

Define the normalized improvement matrix
$M = C_\epsilon^{-1/2} Q\, C_\epsilon^{-1/2}$. Because
\begin{equation*}
  I_D - M = C_\epsilon^{-1/2} R_\epsilon\, C_\epsilon^{-1/2} \succ 0
\end{equation*}
and $M \succeq 0$, the eigenvalues $\mu_1, \dots, \mu_D$ of $M$ satisfy
$0 \le \mu_j < 1$. Factoring $R_\epsilon$ through $C_\epsilon$ gives
\begin{equation}
  \Scomp = -\frac{1}{2} \log_2 \det(I_D - M)
  = -\frac{1}{2} \sum_{j=1}^{D} \log_2 (1 - \mu_j).
  \label{eq:analysis-spectral}
\end{equation}
These eigenvalues measure regularized improvement relative to the target
covariance, with the resolution floor included. Each term is nonnegative, and
a direction with $\mu_j = 0$ contributes zero. All terms vanish exactly when
$Q = 0$, which, because $A \succ 0$, holds if and only if $C_{FZ} = 0$; this
establishes the properties stated in Section~\ref*{sec:gcomp}.

\paragraph{Controlled-covariance analysis.}
Assume full covariance whitening, $C_F = I_m$ and $C_Z = I_D$. This is stronger
than centering or coordinatewise standardization. Then
$Q = \frac{1}{1 + \gamma} C_{FZ}^\top C_{FZ}$. Let $q_1, \dots, q_D$ denote the
eigenvalues of $Q$. Valid joint covariances imply
$0 \le q_j \le 1 / (1 + \gamma)$, with additional rank constraints when
$m < D$. The score is
\begin{equation*}
  \Scomp = \frac{1}{2} \sum_{j=1}^{D} \log_2
  \frac{1 + \epsilon}{1 + \epsilon - q_j},
\end{equation*}
and its marginal rewards and curvatures satisfy
\begin{equation}
  \frac{\partial \Scomp}{\partial q_j}
  = \frac{1}{2 \ln 2\, (1 + \epsilon - q_j)} > 0,
  \qquad
  \frac{\partial^2 \Scomp}{\partial q_j^2}
  = \frac{1}{2 \ln 2\, (1 + \epsilon - q_j)^2} > 0.
  \label{eq:analysis-convex}
\end{equation}
The objective is increasing and convex in each improvement eigenvalue.
Consequently, equal increments in an already-strong direction receive a larger
marginal reward than equal increments in a weak direction.

\paragraph{Fixed-budget consequence.}
Under a fixed total improvement $\sum_j q_j = \tau$, Jensen's inequality gives
\begin{equation}
  \Scomp \ge \frac{D}{2} \log_2
  \frac{1 + \epsilon}{1 + \epsilon - \tau / D}.
  \label{eq:analysis-budget}
\end{equation}
Equality holds for a uniform spectrum whenever it is attainable. Thus, the
uniform spectrum minimizes the score under this constraint, while
concentrating improvement can increase it, subject to eigenvalue and rank
constraints. For example, with $m = D = 2$, $\gamma = 1$, and
$\epsilon \to 0$,
\begin{equation*}
  \Scomp(q_1 = 0.25, q_2 = 0.25) \approx 0.415, \qquad
  \Scomp(q_1 = 0.50, q_2 = 0) = 0.500.
\end{equation*}
Both spectra have total improvement $0.5$, and their residual matrices remain
positive definite in this limit. The concentrated allocation receives the
higher score.

This is a conditional allocation result, not a claim that the objective
necessarily causes dimensional collapse. The score increases with every $q_j$
individually. Concentration becomes relevant when improving one direction
requires sacrificing another, and the actual encoder dynamics depend on which
covariance changes are attainable.

\subsection{Dependence on representation scale}
\label{app:analysis-scale}

Hold $F$ fixed and scale the centered targets by $a > 0$,
$\bar{Z}_\theta \mapsto a \bar{Z}_\theta$. Then $B \mapsto a B$,
$C_Z \mapsto a^2 C_Z$, $Q \mapsto a^2 Q$, and $R \mapsto a^2 R$, so
\begin{equation}
  \Scomp(a \bar{Z}_\theta) = \frac{1}{2} \log_2
  \frac{\det(C_Z + \epsilon a^{-2} I_D)}{\det(R + \epsilon a^{-2} I_D)}.
  \label{eq:analysis-scale}
\end{equation}
When $\epsilon = 0$ and $C_Z$ and $R$ are positive definite, the scale factors
cancel exactly and $\Scomp(a \bar{Z}_\theta) = \Scomp(\bar{Z}_\theta)$. For
fixed $\epsilon > 0$, exact scale invariance no longer holds. Writing
$\delta = \epsilon / a^2$,
\begin{equation}
  \frac{\mathrm{d}\Scomp(a \bar{Z}_\theta)}{\mathrm{d}a}
  = \frac{\epsilon}{a^3 \ln 2}
  \operatorname{tr}\big[ (R + \delta I_D)^{-1} - (C_Z + \delta I_D)^{-1} \big]
  \ge 0.
  \label{eq:analysis-scale-derivative}
\end{equation}
The positive covariance floor therefore introduces a scale preference. The
scaling policy and $\epsilon$ are part of the objective specification, rather
than interchangeable numerical details. These results assume fixed observer
features and an admissible target-rescaling direction; output normalization
can remove that direction from the model's parameterization.

\section{Experimental Details}
\label{app:experiment-details}

\paragraph{Encoders.} Every method trains the same ViT encoder of width $192$,
twelve blocks, three heads, and patch size $16$ on $224 \times 224$ images. The
representation is the $\ell_2$-normalized mean of the final-LayerNorm patch
tokens, excluding the class token, giving a $192$-dimensional code.

\paragraph{Observers.} All observers are frozen, randomly initialized
transformers of width $192$, twelve blocks, and three heads, whose mean-pooled
output gives the observer features. For each example, the masks
$\mathcal{M}_i$ define eight independent random views, each leaving $49$ of
the $196$ patch positions visible; the views may overlap and are drawn from a
shared deterministic mask stream that is identical across methods. The
$\Scomp$ observer receives $8 \times 49 + 28 + 28 = 448$ tokens: view tokens,
relation tokens built from the concatenated binary visibility maps of each of
the $28$ view pairs, and interaction tokens built from the zero-masked raw RGB
views $u_i, u_j$ (Section~\ref*{sec:comp-observer}). Each token carries a fixed
type embedding, and pair tokens carry the sum of their two view positions. The
$\Sphicomp$ observer receives the $392$ view tokens only. The $S^\phi$
observer receives the $196$ whole-image patch tokens and
uses the reference conventions of \citet{zhang2026learnable}: standardized
features (\appref{app:preprocessing}), a sum-error ridge with $\lambda = 3$,
and $\eta = 30$. Observer weights are never updated, and encoder outputs never
enter the observer's feature construction.

\paragraph{Training.} The three epiplexity objectives use AdamW~\citep{loshchilov2019decoupled} with
$\beta = (0.9, 0.999)$, weight decay $0.01$, gradient clipping at norm $1$, and
a $334$-update linear warmup, for ten epochs of $3{,}336$ updates at batch size
$384$. The learning rate is $3 \times 10^{-4}$ for $S^\phi$ and $\Sphicomp$ and
$10^{-3}$ for $\Scomp$. The constants $\gamma$, $\eta$, and $\epsilon$ of
$\Sphicomp$ and $\Scomp$ are calibrated separately for each seed on
unlabeled training batches and frozen before training. $S^\phi$ uses its
reference image pipeline: bilinear resizing to $224 \times 224$,
without augmentation or ImageNet normalization. Encoder and observer
computation uses bfloat16 autocast, and ridge solves and scores are computed in
double precision.

\paragraph{Hardware.} Each training run executed on a single NVIDIA GPU, using
H100, H200, B200, and B300 GPUs across RunPod and Lambda instances. The compute
profiles below were measured on a B300.

\paragraph{Evaluation protocols.} The ImageNet linear probe is trained with
LARS for $90$ epochs at batch size $16{,}384$ (base learning rate $6.4$, ten
warmup epochs, cosine schedule) on frozen features of the full training set and
evaluated on the $50{,}000$ validation images; $5$-nearest-neighbor accuracy
uses the same features. Low-label ImageNet fits a double-precision L-BFGS
multinomial logistic regression on $20{,}000$ labeled training images ($20$ per
class), selects its $\ell_2$ penalty on $5{,}000$ development images, and tests
on $10{,}000$ validation images. Transfer uses the same protocol with $800$
fitting and $200$ development images, and $10{,}000$ (CIFAR-100), $1{,}880$
(DTD), and $2{,}700$ (EuroSAT) test images. Relational readouts are multilayer
perceptrons with hidden width $256$, trained for up to $30$ epochs with the
epoch and weight decay selected on development data. SpatialSense predicts a
binary relation label from encoder features, object names, and normalized
object locations ($11{,}238$ fitting, $2{,}638$ development, and $3{,}622$ test
examples); the CLEVR subset contains questions whose program includes a
\emph{relate} operation, answered over $27$ classes from encoder features and
TF-IDF question features ($23{,}679$, $5{,}924$, and $5{,}912$ examples), with
image-disjoint splits. Spread is computed from the shares $p_j$ of variance in
each principal direction of the low-label test features: the effective rank of
the representation covariance~\citep{roy2007effective} is $\exp(-\sum_j p_j \ln p_j)$ and the isotropy is $1 / (D \sum_j p_j^2)$.
Held-out predictability fits a ridge readout from a common frozen compositional
observer on $384$ support images and predicts the codes of $384$ disjoint query
images, relative to predicting the support mean.
Figure~\ref{fig:stage2-probe-history} shows ImageNet linear-probe accuracy over
the $90$ epochs of probe training.

\begin{figure}[H]
\centering
\includegraphics[width=0.6\linewidth]{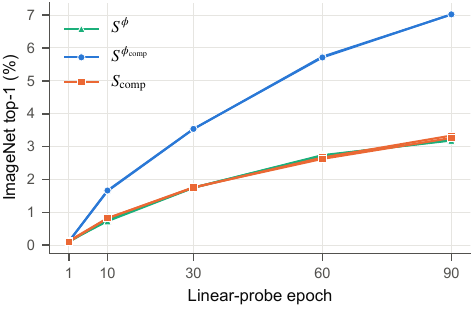}
\caption{ImageNet linear-probe top-1 accuracy during probe training, one line per
seed.}
\label{fig:stage2-probe-history}
\end{figure}

\paragraph{Training compute and comparison with MAE.} We profile one complete
update of each method on a batch of $384$ images and extrapolate to the
$33{,}360$ updates of a run, counting two FLOPs per multiply-add.
Table~\ref{tab:stage2-compute} compares the three epiplexity objectives with a
masked autoencoder (MAE) trained on the same data and schedule with the same
encoder and an eight-block, $512$-wide decoder. All three objectives are far
cheaper to train: MAE requires $434.78$ PFLOPs per run, $3.41\times$ the
counted compute of $S^\phi$, $2.58\times$ that of $\Sphicomp$, and $2.24\times$ that of $\Scomp$ (Figure~\ref{fig:stage2-compute}). Over two seeds, MAE uses
$869.56$ PFLOPs, compared with $255.24$, $337.28$, and $387.63$ PFLOPs for the
three epiplexity objectives. The profiler omits operations it does not support,
including softmax, normalization, Cholesky and triangular solves,
convolutions, and optimizer kernels, so the counts are lower bounds.

\begin{table}[H]
\caption{Counted training compute, extrapolated from one profiled update. The last
column gives MAE's compute as a multiple of each method's.}
\label{tab:stage2-compute}
\centering
\small
\begin{tabular}{lcccc}
\toprule
Method & TFLOPs/update & PFLOPs/run & PFLOPs (two seeds) & MAE $\times$ \\
\midrule
MAE & 13.033 & 434.78 & 869.56 & 1.00 \\
$S^\phi$ & 3.826 & 127.62 & 255.24 & 3.41 \\
$\Sphicomp$ & 5.055 & 168.64 & 337.28 & 2.58 \\
$\Scomp$ & 5.810 & 193.82 & 387.63 & 2.24 \\
\bottomrule
\end{tabular}
\end{table}

\begin{figure}[H]
\centering
\includegraphics[width=0.78\linewidth]{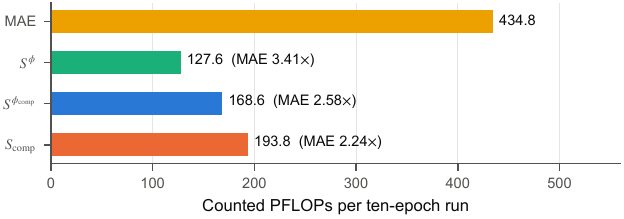}
\caption{Counted PFLOPs per ten-epoch run, with MAE's multiple of each method.
Counts are lower bounds.}
\label{fig:stage2-compute}
\end{figure}

\paragraph{Per-seed results.} Table~\ref{tab:stage2-per-seed} reports every
accuracy in Table~\ref*{tab:stage2-accuracy} for each pretraining seed, and
Figure~\ref{fig:stage2-per-seed} plots them.

\begin{table}[H]
\caption{Frozen-encoder accuracy (\%) for each seed.}
\label{tab:stage2-per-seed}
\centering
\small
\setlength{\tabcolsep}{5pt}
\begin{tabular}{lcccccc}
\toprule
& \multicolumn{2}{c}{$S^\phi$} & \multicolumn{2}{c}{$\Sphicomp$} & \multicolumn{2}{c}{$\Scomp$} \\
\cmidrule(lr){2-3} \cmidrule(lr){4-5} \cmidrule(lr){6-7}
Evaluation & Seed 0 & Seed 1 & Seed 0 & Seed 1 & Seed 0 & Seed 1 \\
\midrule
ImageNet top-1 & 3.20 & 3.18 & 7.02 & 7.02 & 3.33 & 3.26 \\
ImageNet top-5 & 9.57 & 9.56 & 17.35 & 17.07 & 9.66 & 9.62 \\
ImageNet 5-NN & 1.64 & 1.59 & 2.67 & 2.68 & 1.49 & 1.46 \\
Low-label ImageNet & 1.91 & 1.96 & 2.23 & 2.23 & 2.08 & 2.05 \\
CIFAR-100 & 8.97 & 7.67 & 9.57 & 9.79 & 6.55 & 7.32 \\
DTD & 11.54 & 10.90 & 11.76 & 10.64 & 11.28 & 11.38 \\
EuroSAT & 63.33 & 63.67 & 67.37 & 68.67 & 65.78 & 65.59 \\
SpatialSense & 52.37 & 53.31 & 52.87 & 52.62 & 52.93 & 51.71 \\
CLEVR relational subset & 44.25 & 44.74 & 44.30 & 44.59 & 44.79 & 43.74 \\
\bottomrule
\end{tabular}
\end{table}

\begin{figure}[H]
\centering
\includegraphics[width=\linewidth]{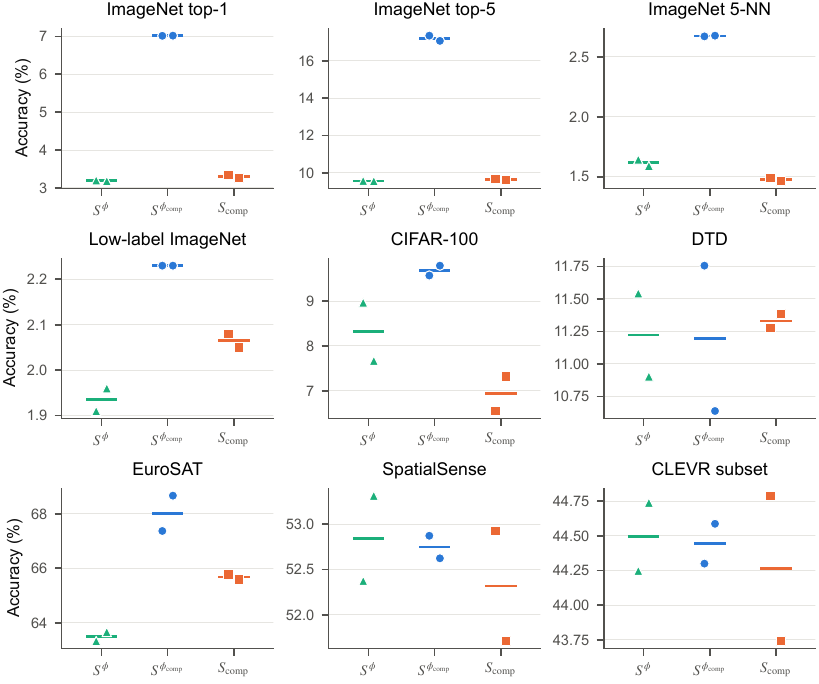}
\caption{Accuracy of each seed on every evaluation; lines mark the two-seed means.}
\label{fig:stage2-per-seed}
\end{figure}

\paragraph{Transfer geometry.} Table~\ref{tab:transfer-geometry} reports the
isotropy of each frozen encoder on the test images of each dataset, and
Figure~\ref{fig:stage2-spectra-datasets} shows the corresponding spectra.

\begin{table}[H]
\caption{Isotropy of frozen codes on each dataset, averaged over two seeds.}
\label{tab:transfer-geometry}
\centering
\small
\begin{tabular}{lcccc}
\toprule
Encoder & ImageNet (low-label) & CIFAR-100 & DTD & EuroSAT \\
\midrule
$S^\phi$ & 0.194 & 0.143 & 0.134 & 0.027 \\
$\Sphicomp$ & 0.869 & 0.563 & 0.456 & 0.039 \\
$\Scomp$ & 0.112 & 0.091 & 0.080 & 0.020 \\
\bottomrule
\end{tabular}
\end{table}

\begin{figure}[H]
\centering
\includegraphics[width=\linewidth]{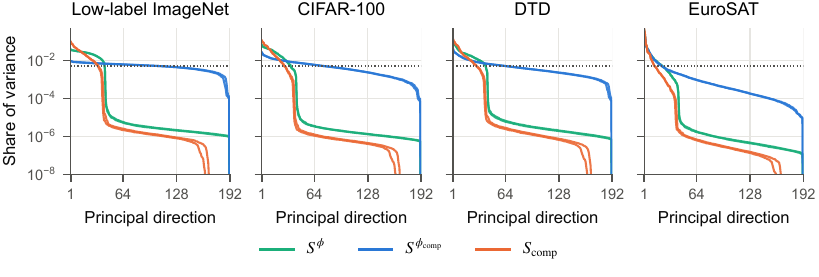}
\caption{Variance share per principal direction on each evaluation dataset, one
line per seed. Dotted: isotropic share $1/192$.}
\label{fig:stage2-spectra-datasets}
\end{figure}

\paragraph{Relational readout controls.} Table~\ref{tab:relational-controls}
reports the readouts trained without encoder features or without metadata,
alongside the encoder-plus-metadata readouts of
Table~\ref*{tab:stage2-accuracy}.

\begin{table}[H]
\caption{Relational readout accuracy (\%) by input, averaged over two seeds.}
\label{tab:relational-controls}
\centering
\small
\begin{tabular}{llccc}
\toprule
Dataset & Readout input & $S^\phi$ & $\Sphicomp$ & $\Scomp$ \\
\midrule
SpatialSense & Encoder and metadata & 52.84 & 52.75 & 52.32 \\
& Encoder only & 51.79 & 50.37 & 51.01 \\
& Object names only & 54.90 & 54.67 & 54.90 \\
& Locations only & 66.97 & 67.06 & 66.97 \\
& Names and locations & 57.28 & 57.01 & 57.28 \\
\midrule
CLEVR subset & Encoder and question & 44.49 & 44.44 & 44.27 \\
& Question only & 43.54 & 43.54 & 43.54 \\
\bottomrule
\end{tabular}
\end{table}

\end{document}